\documentclass[letterpaper, 10pt, conference]{ieeeconf}

\makeatletter

\let\proof\@undefined
\let\endproof\@undefined
\makeatother

\IEEEoverridecommandlockouts  

\usepackage{cite}
\usepackage{amsmath,amssymb,amsfonts}
\usepackage{algorithmic}
\usepackage{graphicx}
\usepackage{textcomp}
\usepackage{xcolor}
\usepackage{booktabs}
\usepackage{multirow}
\usepackage{xspace}  
\usepackage{amsthm}
\usepackage{amsfonts}
\usepackage{mathtools}  
\usepackage{bm}  
\usepackage{paralist}
\usepackage{soul}
\usepackage[nolist,nohyperlinks]{acronym}
\acrodef{cbf}[CBF]{Control Barrier Function}
\acrodef{mpc}[MPC]{Model Predictive Control}
\acrodef{qp}[QP]{Quadratic Program}
\acrodef{sdf}[SDF]{Signed Distance Function}
\usepackage[hidelinks,bookmarks]{hyperref} 

\usepackage{multicol}
\usepackage{citesort}
\usepackage[textsize=tiny]{todonotes}
\usepackage{siunitx}
\usepackage{relsize}
\theoremstyle{definition}
\newtheorem{definition}{Definition}
\newtheorem{theorem}{Theorem}
\newtheorem{lemma}{Lemma}
\newtheorem{proposition}{Proposition}
\theoremstyle{remark}
\newtheorem{remark}{Remark}

\title{\LARGE \bf Safe Quadrotor Navigation using Composite Control Barrier Functions}
\author{Marvin Harms, Martin Jacquet, Kostas Alexis
	\thanks{Autonomous Robots Lab, Norwegian University of Science and Technology (NTNU), Trondheim, Norway,
    {\tt \footnotesize
        \href{mailto:marvin.c.harms@ntnu.no}{marvin.c.harms@ntnu.no}}
    }
    \thanks{This work was supported by the European Commission Horizon Europe grants DIGIFOREST (EC 101070405) and SPEAR (EC 101119774).}
}

\begin{document}
\maketitle

\begin{abstract}
This paper introduces a safety filter to ensure collision avoidance for multirotor aerial robots. The proposed formalism leverages a single Composite Control Barrier Function from all position constraints acting on a third-order nonlinear representation of the robot's dynamics. We analyze the recursive feasibility of the safety filter under the composite constraint and demonstrate that the infeasible set is negligible. The proposed method allows computational scalability against thousands of constraints and, thus, complex scenes with numerous obstacles. We experimentally demonstrate its ability to guarantee the safety of a quadrotor with an onboard LiDAR, operating in both indoor and outdoor cluttered environments against both naive and adversarial nominal policies.
\end{abstract}

\section{Introduction}

Autonomous aerial robots, such as multirotor platforms, are widely used in scientific and industrial contexts alike. A diverse set of works has focused on solving the core technological challenges among which safe autonomous navigation in unknown environments is of paramount importance. A multitude of methods for collision avoidance exists that include reactive schemes~\cite{alejo2016reactive}, control strategies accounting for the presence of obstacles as constraints~\cite{dentler2016real,alexis2016robust}, learning~\cite{kahn2017uncertainty,kulkarni2024reinforcement}, and map-based motion planning~\cite{karaman2011sampling,allen2016real}. Despite the unprecedented progress, ensuring the safe autonomous navigation of aerial robots in complex environments remains a pertinent challenge.

Control-driven methods and safety filters represent an appealing strategy for collision-free safe flight in that they correspond to a key layer of the autonomy stack that typically runs at high update rates while integrating the system's dynamics natively. Potentially working in synergy with other approaches, such as map-based motion planning, they represent a formal last-resort safety mechanism even when higher-level methods fail, e.g. due to global map drift or planning at low rates. Successful examples of control-driven methods applied to multirotor collision avoidance are \ac{mpc} and \acp{cbf}. 

However, control-oriented methods to ensure collision avoidance of aerial robots often struggle to scale to cluttered and obstacle-rich environments due to the number or complexity of the imposed constraints, which becomes prohibitive for high-rate control onboard computationally constrained platforms~\cite{alexis2016robust}. Safety filters based on \acp{cbf} often pose a computationally cheap alternative to \ac{mpc}. Nevertheless, finding an admissible \ac{cbf} for high-order systems is a challenging task.

Aiming to overcome these challenges, the contribution of this work is two-fold. First, we propose a computationally scalable safety filter using a composite \ac{cbf} representing all collision constraints. Second, the proposed approach is formulated for a third-order nonlinear system, which allows the formulation of a safety filter without requiring dynamic inversion of the nonlinear attitude dynamics, and a better representation of the actual system constraints such as positive thrust values.
The method is first validated by demonstrating its scalability against the order of magnitude of the number of position constraints. The safety filter is then tested in hardware experiments, indoors and outdoors, both when guided by a naive policy that is agnostic to obstacles and when guided by an adversarial policy trying to collide with the environment.

The remainder of this work is structured as follows: Sec.~\ref{sec:relatedwork} presents an overview of the related literature, while Sec.~\ref{sec:preliminaries} provides an overview on \ac{cbf} theory. Sec.~\ref{sec:approach} introduces the problem statement, proposed method, and a brief analysis thereof. The experimental results are detailed in Sec.~\ref{sec:evaluation} and concluding remarks are presented in Sec.~\ref{sec:conclusions}.

\section{Related Work}\label{sec:relatedwork}

Due to the inherent safety concern in aerial robots, multiple prior works have considered safe control for multirotors using \acp{cbf}. One earlier work employs multiple exponential \acp{cbf} inside a safety filter \ac{qp} for a linear, second-order system to obtain safe acceleration commands for a double integrator. In \cite{cascaded_CBF}, exponential \acp{cbf} are again used to formulate a safety filter using two cascaded Quadratic Programs (QP) for thrust and torque constraints. The cascaded structure is reported to lead to feasibility issues. In \cite{range_sensing_CBF}, the authors design a safe control law for a multirotor using \acp{cbf} with configuration constraints and collision constraints. This approach only manipulates thrust for obstacle avoidance, while attitude is only altered to satisfy the configuration constraint.
The authors of \cite{cbf_potential_field_analysis} perform a comparative analysis of \acp{cbf} and potential fields for obstacle avoidance, demonstrating results in a 2D planar quadrotor using a safety filter.
In the recent work \cite{backstepping_CBF}, a \ac{cbf} backstepping approach is used to impose multiple, convex constraints for position, velocity and rates for a multirotor in simulation. In \cite{collisionConeCBF}, a novel type of \ac{cbf} is proposed for the avoidance of dynamic obstacles. However, this formulation disallows directly approaching obstacles, which leads to feasibility issues in highly cluttered environments.
In \cite{compositeCBFtemporal}, the authors propose using the recently introduced composite \acp{cbf} for control barrier synthesis in unknown environments, demonstrating the applicability in simulation on a linear, double integrator model of a quadrotor in a 2D map.
In \cite{CBF_aided_teleop}, a teleoperation scheme for safe position control of a quadrotor, modelled as a linear system, is presented, where the local occupancy map is represented as an \ac{sdf} and used as a discrete-time \ac{cbf}.
Methods for learning \acp{cbf} have also been recently proposed to achieve collision avoidance. Here, the sampling-based approach shown in \cite{dawson2022learning} demonstrate safe control for a first order systems, while the work in \cite{harms2024neural} has been applied to a quadrotor modelled as a double integrator but has yet to demonstrated for systems of order 3 and higher.

\section{Preliminaries}\label{sec:preliminaries}
\subsection{Notation}
\begin{tabbing}
 \hspace*{2.2cm} \= \kill
  $\mathbf{x} \in \mathcal{X},  \mathbf{u} \in \mathcal{U}$ \>  state and input vectors \\[0.5ex]
  $\mathcal{L}_f h, \mathcal{L}_g h \mathbf{u}$ \>  Lie derivatives of $h$ along $f$, $g \mathbf{u}$ \\[0.5ex] 
  $\frac{d}{dt} V$, $\nabla_\mathbf{x} V$ \>  time derivative of $V$, gradient of $V$ w.r.t. $\mathbf{x}$ \\[0.5ex] 
  $\mathbf{a} \cdot \mathbf{b}$, $\mathbf{a} \times \mathbf{b}$ \>  dot product, cross product of $\mathbf{a}$ and $\mathbf{b}$ \\[0.5ex]  
  $\|\cdot\|$ \>  Euclidean norm \\[0.5ex] 
  $[\mathbf{a}]_{\times}$ \>  skew-symmetric matrix associated with $\mathbf{a}$ \\[0.5ex]
\end{tabbing}

\subsection{Control Barrier Functions}
Let us consider the general, control affine system

\small
\begin{equation}\label{system}
    \frac{d}{dt} \mathbf{x} = f(\mathbf{x}) + g(\mathbf{x}) \mathbf{u},
\end{equation}
\normalsize
where $\mathbf{x} \in \mathcal{X} \subseteq \mathbb{R}^n$ and $\mathbf{u} \in \mathcal{U} \subseteq \mathbb{R}^m$. We now consider a continuously differentiable function $h: \mathcal{X} \rightarrow \mathbb{R}$ with the property $\{\mathbf{x} | \frac{dh}{d\mathbf{x}}(\mathbf{x}) = 0\} \cap \{\mathbf{x} | h(\mathbf{x}) = 0\} = \emptyset$, which describes the superlevel set $\mathcal{X}_\text{safe} = \{\mathbf{x} \in \mathcal{X}  :  h(\mathbf{x}) \geq 0 \}$.

\begin{definition}[Control Invariant Set]
    A set $\mathcal{X}_\text{safe}$ is a forward control invariant set for the system \eqref{system}, if for any $\mathbf{x}({t_0}) \in \mathcal{X}_\text{safe}$ there exists at least one input trajectory $\mathbf{u}(t) \in \mathcal{U}$ such that $\mathbf{x}({t}) \in \mathcal{X}_\text{safe} \quad \forall t\geq t_0$ under the system \eqref{system}.
\end{definition}

\begin{definition}[Control Barrier Function \cite{ames2019control}]\label{def:cbf}
    Let $\mathcal{X}_\text{safe}$ be the $0$-superlevel set of a continuously differentiable function
    $h: \mathcal{X} \rightarrow \mathbb{R}$ with the property that $\mathcal{X}_\text{safe} = \{ \mathbf{x} | h(\mathbf{x})\geq 0 \}$. Then, $h$ is a \ac{cbf} if there exists an extended class $\mathcal{K}_\infty$ function $\alpha(\cdot)$ such that for the system \eqref{system} it holds:

    \small
    \begin{equation}\label{CBF_lie_condition}
        \underset{\mathbf{u} \in \mathcal{U}}{\text{sup}} [\mathcal{L}_f h(\mathbf{x}) + \mathcal{L}_g h(\mathbf{x}) \mathbf{u}] \geq -\alpha(h(\mathbf{x}))
    \end{equation}
    \normalsize
    for all $\mathbf{x} \in \mathcal{X}_\text{safe}$. Further, an input $\mathbf{u}$ is considered safe with respect to a valid CBF, if it satisfies \eqref{CBF_lie_condition}.
\end{definition}

From the above, it becomes clear that $\mathcal{X}_\text{safe}$ is a control invariant set for the system \eqref{system} subject to any control law satisfying \eqref{CBF_lie_condition}. See \cite{ames2016control} for a proof. Condition \eqref{CBF_lie_condition} requires a nonzero Lie derivative, e.g. $\mathcal{L}_g h(\mathbf{x}) \neq 0$ in general, restricting the use of \acp{cbf} to functions $h$ of relative degree~1. Systematic approaches to construct a CBF with relative degree greater than one include exponential control barrier functions \cite{ECBF}, high-order control barrier functions \cite{HOCBF} and backstepping control barrier functions \cite{BCBF}.

\subsection{Exponential Control Barrier Functions}
Consider a safety metric $h(\mathbf{x})$ of uniform relative degree $r\geq1$. Repeatedly differentiating $h(\mathbf{x})$ with respect to time results in terms $\mathcal{L}_g \mathcal{L}_f^{i} h(\mathbf{x})$ which are equal to zero for $i<r-1$ due to the relative degree assumption. However, by using negative constants (poles) $p_i<0$ and defining the series of functions $\nu_i: \mathcal{X} \rightarrow \mathbb{R}$ and corresponding superlevel sets $C_i$ for $i \in \{0,1, ...,r\}$ as

\small
\begin{align*}
    \nu_0(\mathbf{x}) &= h(\mathbf{x}), & C_0 &= \{ \mathbf{x} : \nu_0(\mathbf{x}) \geq 0 \}, \\
    \nu_1(\mathbf{x}) &= \dot{\nu}_0(\mathbf{x}) - p_1 \nu_0(\mathbf{x}), & C_1 &= \{ \mathbf{x} : \nu_1(\mathbf{x}) \geq 0 \}, \\
    &\vdots & &\vdots \\
    \nu_r(\mathbf{x}) &= \dot{\nu}_{r-1}(\mathbf{x}) - p_r \nu_{r-1}(\mathbf{x}), & C_r &= \{ \mathbf{x} : \nu_r(\mathbf{x}) \geq 0 \},
\end{align*}
\normalsize
it is shown in \cite{ECBF} that the following theorem holds

\begin{theorem}\cite{ames2019control}\label{thm1}
If $C_r$ is forward-invariant and $\mathbf{x}_0 \in \bigcap_{i=0}^r C_i$ then $\mathcal{C}_0$ is forward-invariant.
\end{theorem}
Note that Theorem \ref{thm1} additionally requires conditions of the initial state $\mathbf{x}_0$ to hold in addition to the invariance of $\mathcal{C}_r$ to ensure invariance of $\mathcal{C}_0$. If the function $\nu_r(\mathbf{x})$ is a \ac{cbf}, then $h(\mathbf{x})$ is said to be an exponential \ac{cbf} (ECBF).

\subsection{Composite Control Barrier Functions}
Enforcing multiple safety constraints simultaneously on a system can lead to practical challenges, especially if the number of constraints considered becomes large. In \cite{compositeCBFames}, a method to construct a single CBF as a logical AND-OR composition from multiple, different safety constraints through a softmin/softmax is described. This approach was also described in \cite{compositeCBFhoagg} for high-order systems with mixed relative degree. For the sake of brevity, we recite the method focusing on ECBFs. We consider a set of functions $\{h_i(\mathbf{x})\}_{i=1}^N$, where $h_i$ is an ECBF of relative degree $r_i$. Furthermore, we assume the safe set $\mathcal{S} = \bigcap_{i=1}^{N} \{\mathbf{x} : h_i(\mathbf{x}) \geq 0\}$ is nonempty. Defining the intermediate high-order functions

\small
\begin{equation}\label{ccbf1}
    \nu_{i,j+1} = \mathcal{L}_f \nu_{i,j} - p_{i,j} \nu_{i,j} ,
\end{equation}
\normalsize
where $\nu_{i,0} = h_i$, $j \in \{0,1, \hdots, r_{i-1} \}$ and $p_{i,j}<0$, we define the sets

\small
\begin{equation}\label{ccbf2}
    \mathcal{C}_{i,j} = \{\mathbf{x} : \nu_{i,j}(\mathbf{x}) \geq 0\} \quad \text{,} \quad C_i = \bigcap_{j=0}^{r_i-1} \mathcal{C}_{i,j}.
\end{equation}
\normalsize
By application of Theorem \ref{thm1} and Definition~\ref{def:cbf} we arrive at the following lemma:
\begin{lemma}\label{lemma1}
If $\mathbf{x}_0 \in \mathcal{C}_i$ and $\nu_{i,r_{i-1}}$ satisfies $\|\mathcal{L}_g \nu_{i,r_{i-1}}(\mathbf{x})\|>0~ \forall \mathbf{x} \in \mathcal{X}$, then the condition with $\alpha > 0$

\small
\begin{equation}\label{ccbf3}
    \mathcal{L}_f \nu_{i,r_{i-1}}(\mathbf{x}) + \mathcal{L}_g \nu_{i,r_{i-1}}(\mathbf{x}) \mathbf{u} \geq -\alpha \nu_{i,r_{i-1}}(\mathbf{x}) \quad \forall t \geq t_0
\end{equation}
\normalsize
implies $\mathbf{x} \in \mathcal{C}_i \quad~ \forall t \geq t_0$ and thus $h_i(\mathbf{x})\geq0$. Also, $\nu_{i,r_{i-1}}$ is a CBF of relative degree~1.
\end{lemma}

Our objective is to satisfy all constraints jointly, e.g. $\mathbf{x} \in \mathcal{S}$, which can be achieved by ensuring $\mathbf{x} \in \mathcal{C} = \bigcap_{i=i}^{N} \mathcal{C}_{i}$. In~\cite{compositeCBFames}, the intersection set $\mathcal{C}$ is compactly expressed through the $\min$ operation as $\{\mathbf{x} : \nu_{i,j}(\mathbf{x}) \geq 0\} = \{\mathbf{x} : \min_i \nu_{i,r_{i}-1}(\mathbf{x}) \geq 0\}$. Correspondingly, a new function using the soft minimum is defined in \cite{compositeCBFames} as a smooth under-approximation of the set $\mathcal{C}$ as

\begin{equation}\label{ccbf}
    h(\mathbf{x}) = - \frac{1}{\kappa} \log \sum_i e^{-\kappa \nu_{i,r_{i-1}}(\mathbf{x})},
\end{equation}
where we have $\{\mathbf{x} : h(\mathbf{x}) \geq 0\} \subseteq \mathcal{C}$ with the parameter $\kappa > 0$. The Lie derivatives of $h(\mathbf{x})$ are given by the functions

\small
\begin{equation}
    \mathcal{L}_f h(\mathbf{x}) = \sum_i \lambda_i(\mathbf{x}) \mathcal{L}_f h_i(\mathbf{x}) \text{,} \quad
    \mathcal{L}_g h(\mathbf{x}) = \sum_i \lambda_i(\mathbf{x}) \mathcal{L}_g h_i(\mathbf{x}),
\end{equation}
\normalsize
where $\lambda_i(\mathbf{x}) = e ^ {-\kappa (h_i(\mathbf{x}) - h(\mathbf{x}) )}$.

\begin{theorem}
The function given by \eqref{ccbf} is a CBF for the system \eqref{system} if and only if~\cite{compositeCBFames}

\small
\begin{equation}\label{theorem3_eq}
    \sum_i \lambda_i(\mathbf{x}) \mathcal{L}_g h_i(\mathbf{x}) = 0 \Longrightarrow \sum_i \lambda_i(\mathbf{x}) (\mathcal{L}_f h_i(\mathbf{x}) + \alpha  h_i(\mathbf{x})) \geq 0
\end{equation}
\normalsize
holds for $\alpha > 0$ and for all $\mathbf{x} \in \mathcal{X}$.
\end{theorem}

\begin{remark}\label{rmrk1}
By Lemma \ref{lemma1}, we have that $\nu_{i,r_{i-1}}(\mathbf{x})$ is a CBF and thus \eqref{theorem3_eq} trivially holds for $\lambda_p = 1$, $\lambda_i = 0 \quad \forall i \in \{1, \hdots ,N \} \setminus p$. This can be achieved by letting $\kappa \rightarrow \infty$ whenever $\nu_{p,r_{i-1}}(\mathbf{x}) < \nu_{i,r_{i-1}}(\mathbf{x}) \quad \forall i \in \{1, \hdots ,N \} \setminus p$. In practice, a large value of $\kappa$ is generally chosen to avoid evaluation of \eqref{theorem3_eq}. 
\end{remark}


\section{Composite CBF for Multirotor Collision Avoidance}\label{sec:approach}
In the following section, we propose a safety filter designed to ensure safety of the system under multiple state constraints. The intersection of control invariant set is in general not control invariant and imposing multiple \acp{cbf} across a control structure may lead to feasibility issues as reported in prior work \cite{cascaded_CBF} in the case of multirotors. Contrasting to prior work, we explicitly analyse the feasibility, showing that the infeasible set is a null-set on $\mathbb{R}^n$.

In the following, we consider the system model of a rate-controlled planar multirotor as in \cite{backstepping_CBF}, described by

\vspace{-3ex}
\small
\begin{subequations}\label{dynamics}
\begin{align}
    \dot x &= v, \label{position_dynamics} \\
    \dot v &= ge_3 - \frac{T}{m} R e_3, \label{linvel_dynamics} \\
    \dot R &= R[\Omega]_{\times}, \label{attitude_dynamics} \\
    \dot T &= \tau, \label{angvel_dynamics}
\end{align}
\end{subequations}
\normalsize
where $g$ is the intensity of gravity, and the artificial control dynamics $\dot T = \tau$ is introduced to achieve a uniform relative degree $r=3$ of any function of $x$ in the input variables $\mathbf{u} = [\Omega^T, \tau]^T$.
\begin{figure}
    \centering
    \includegraphics[width=0.35\linewidth]{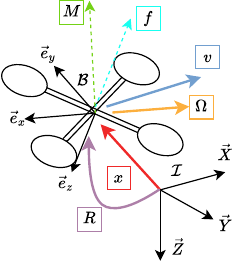}
    \hspace{1cm}
    \includegraphics[width=0.4\linewidth]{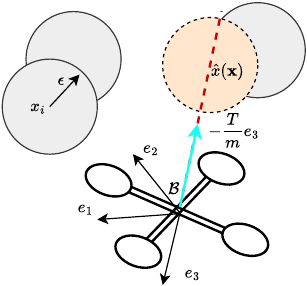}
    \caption{\small \textit{Left}: Coordinate conventions utilized. All equations are expressed in North-East-Down (NED) and Front-Right-Down (FRD) frame. \textit{Right}: Singular configuration example. A necessary requirement for conflicting constraints in \eqref{OCP} is that the virtual obstacle (orange) lies on the positive thrust axis (shown in red).}
    \label{fig:conventions}
    \vspace{-6ex}
\end{figure}
\subsection{Reference Controller}
As a reference controller, we propose a simple adaptation of the commonly used geometric controller \cite{Lee} to account for the change in system dynamics \eqref{dynamics}. The proposed geometric control law becomes

\small
\begin{subequations}
\begin{align}
    \Omega &= -k_R \frac{1}{2}[R_d^tR - R^T R_d]\textsuperscript{$\wedge$} + R^T R_d \Omega_{d}, \label{lee_attitude}\\
    T_d &= - (K_x e_x + K_v e_v + m g e_3 - m \Ddot{x})^T R e_3,\label{lee_position} \\
    \tau &= k_T (T_d - T) + \dot{T}_d, \label{lee_thrust}
\end{align}
\end{subequations}
\normalsize
where $(\cdot)$\textsuperscript{$\wedge$} is the vectorize operator. While it is straightforward to show global exponential stability of the thrust error $T_d - T$, local exponential stability of the attitude error $ \frac{1}{2}\text{tr}[\mathbb{I} - R_d^T R]$ can be shown in fashion similar to \cite{LeeProof}.

\subsection{Constraints}
In this work, we seek to enable safe control of multirotors in arbitrarily cluttered environments. To circumvent representing complex geometries by means of more elaborate functions, we choose to represent the collision constraint as a composition of multiple, simple constraints, describing the free space $\mathcal{C}_x$ as

\small
\begin{equation}\label{collision_constraint}
    \mathcal{C}_x = \bigcap_{i=1}^{N} \{x : \|x - x_i\| \geq \epsilon\} \text{,}
\end{equation}
\normalsize
where $x_i$ represents the obstacle locations and $\epsilon>0$. Furthermore, the model described above does not have any limit on the thrust value $T$. To account for the inability of most multirotors to exhibit negative thrust values, we also consider a thrust constraint set $\mathcal{C}_T$ as

\small
\begin{equation}\label{thrust_constraint}
    \mathcal{C}_T = \{ T : T \geq T_\text{min} \},
\end{equation}
\normalsize
for some positive constant $T_\text{min}\ge0$.
We do not consider additional input bounds on $\mathbf{u}$ and leave this to future work. Our goal is to ensure $x \in \mathcal{C}_x$ and $T \in \mathcal{C}_T$ at all times during operation, i.e. ensure invariance of the set $\mathcal{C}_x \cap \mathcal{C}_T$.

\subsection{Control Barrier Function Design}
To ensure $\mathbf{x} \in \mathcal{C}_x$ with obstacles located at $x_i$ for $i\in \{1, \hdots, N \}$, we define a set of functions as

\begin{equation}\label{position_cbf}
   \nu_{i,0}(\mathbf{x}) = ||(x-x_i)||^2 - \epsilon^2.
\end{equation}

Noting that $\nu_{i,0}$ have relative degree 3 in all control inputs, we define the functions, given two constants $p_0, p_1 < 0$,

\small
\begin{subequations}\label{composition1}
    \begin{equation}
        \nu_{i,1} = \mathcal{L}_f \nu_{i,0} - p_{0} \nu_{i,0},
    \end{equation}
    \begin{equation}
        \nu_{i,2} = \mathcal{L}_f \nu_{i,1} - p_{1} \nu_{i,1}.
    \end{equation}
\end{subequations}
\normalsize

Using lemma \ref{lemma1}, we note that $\nu_{i,2}$ are \acp{cbf} of relative degree 1 ($\|\mathcal{L}_g \nu_{i,2}(\mathbf{x})\|>0$ holds everywhere except for $x=x_i$, which is not a meaningful scenario). Constructing the composite \ac{cbf} as in \eqref{ccbf} with the additional saturation through the $\tanh$ function as proposed in \cite{compositeCBFames} with $\gamma>0$ yields the \ac{cbf} candidate

\small
\begin{equation}\label{composite_cbf}
    h_1(\mathbf{x}) = - \frac{\gamma}{\kappa} \log \sum_i e^{-\kappa 
\tanh{\frac{ \nu_{i,2}(\mathbf{x})}{\gamma}}}.
\end{equation}
\normalsize
To ensure invariance of the set $\mathcal{C}_x$, the standard condition 
\eqref{CBF_lie_condition} with \eqref{composite_cbf} becomes using $\alpha_1 > 0$ and $\lambda_i(\mathbf{x}) = e ^ {-\kappa (\nu_{i,2}(\mathbf{x}) - h_1(\mathbf{x}) )}$:

\small
\begin{equation}\label{composite_invariance}
\begin{split}
    \alpha_1 h_1(\mathbf{x}) +
    \underbrace{\frac{2T}{m}(x- \sum_i \frac{\lambda_i(\mathbf{x})}{\cosh^2(\frac{\nu_{i,2}(\mathbf{x})}{\gamma})} x_i)^T R [  [e_3]_\times, - e_3]}_{\mathcal{L}_g h_1(\mathbf{x})} \cdot \mathbf{u} \geq \\
    - \underbrace{ \sum_i \frac{\lambda_i(\mathbf{x})}{\cosh^2(\frac{\nu_{i,2}(\mathbf{x})}{\gamma})} (\mathcal{L}_f^3 \nu_{i,0} (\mathbf{x}) - p_1\mathcal{L}_f^2 \nu_{i,0} (\mathbf{x}) + p_1 p_0\mathcal{L}_f \nu_{i,0} (\mathbf{x}))}_{\mathcal{L}_f h_1(\mathbf{x})} \text{.}
\end{split}
\end{equation}
\normalsize
The structure of the above constraint regressor $\mathcal{L}_g h(\mathbf{x})$ allows the interpretation of a ``virtual obstacle'' at the location 

\small
\begin{equation}
    \hat{x}(\mathbf{x}) =  \sum_i \frac{\lambda_i(\mathbf{x})}{\cosh^2(\frac{\nu_{i,2}(\mathbf{x})}{\gamma})} x_i.
\end{equation}
\normalsize
The relative distance vector $(x- \hat{x}(\mathbf{x}))$ can be seen as the virtual obstacle direction, constraining the direction of the control input $\mathbf{u}$.
The virtual obstacle is therefore a weighted average of the original constraints $\nu_{i,0}$
with weighting terms $\lambda_i (\mathbf{x})$. Furthermore, the constraint regressor above is nonzero as long as $x \neq \hat{x}(\mathbf{x})$ since the column vectors of $ [[e_3]_\times, - e_3]$ form a basis of $\mathbb{R}^3$. We therefore have that $h(\mathbf{x})$ is a \ac{cbf} from (\ref{theorem3_eq}) iff $x = \hat{x}(\mathbf{x}) \Longrightarrow \mathcal{L}_f h(\mathbf{x}) + \alpha  h(\mathbf{x}) \geq 0$. While an online evaluation of this condition is impractical due to the non-fixed and arbitrary configuration of obstacles, we resort to the argument taken in Remark~\ref{rmrk1}, i.e. that for one $i$, $\underset{\kappa \rightarrow \infty}{\lim} h_1(\mathbf{x}) = \gamma \tanh({{ \nu_{i,2}(\mathbf{x})}/{\gamma}})$, while $\nu_{i,2}$ is a \ac{cbf}.

To ensure the thrust constraint by invariance of $\mathcal{C}_T$, we introduce the \ac{cbf} candidate with relative degree 1

\small
\begin{equation}\label{thrust_cbf}
    h_2(\mathbf{x}) = T - \epsilon_T.
\end{equation}
\normalsize
The corresponding invariance condition \eqref{CBF_lie_condition} becomes with  $\alpha_2 > 0$

\small
\begin{equation}\label{thrust_invariance}
     \underbrace{[\mathbb{O}^{1\times3}, 1]}_{\mathcal{L}_g h_2\mathbf{x})} \cdot \mathbf{u} \geq \underbrace{-\alpha_2 h_2 (\mathbf{x}) - \mathcal{L}_f h_2(\mathbf{x})}_{b_2(\mathbf{x})}.
\end{equation}
\normalsize
Looking at the constraint regressor, we see that $\mathcal{L}_g h(\mathbf{x})\neq 0 \quad \forall \mathbf{x} \in \mathcal{X}$, so \eqref{thrust_cbf} indeed is a \ac{cbf}.

We now proceed with a feasibility analysis to show that the stacking of the constraints only has an insignificant violation space, which is given by a zero-volume null set.

\subsection{Recursive Feasibility}

We analyze the feasibility of the constraints \eqref{composite_invariance} and \eqref{thrust_invariance} by considering the constraint regressors $\mathcal{L}_g h_1(\tilde{\mathbf{x}}) = [l_1(\mathbf{x}), l_2(\mathbf{x}), 0, l_3(\mathbf{x})]$ and $\mathcal{L}_g h_2(\mathbf{x})$. To result in conflicting constraints, we require $\mathcal{L}_g h_1(\mathbf{x}) \cdot \mathcal{L}_g h_2(\mathbf{x})^T <0$ which is true only when 

\small
\begin{equation}\label{infeasibility1}
    \frac{2T}{m}(x- \hat{x}(\mathbf{x}))^T R e_3 >0 \text{,}
\end{equation}
\normalsize
if $T > 0$. Thus, the virtual obstacle must be 'above' the quadrotor. Furthermore, the constraints can only conflict if

\small
\begin{equation}\label{infeasibility2}
    (x- \hat{x}(\mathbf{x})) \times R e_3 = 0^{3 \times 1}
\end{equation}
\normalsize
since otherwise the constraint sensitivity on roll and pitch in \eqref{composite_invariance} is nonzero. The conflicting cases (denoted as $\tilde{\mathbf{x}}$) therefore require that the thrust vector $-T e_3$ and the virtual obstacle-relative vector $(x- \hat{x}(\mathbf{x}))$ are colinear. In this case, we have $\mathcal{L}_g h_1(\tilde{\mathbf{x}}) = -\beta \mathcal{L}_g h_2(\tilde{\mathbf{x}})$ for $\beta\geq 0$ and thus $l_1(\mathbf{x})=0$, $l_2(\mathbf{x})=0$. This configuration is illustrated in Fig. \ref{fig:conventions}.

By studying the sensitivity of constraint \eqref{infeasibility2} for the idealized case (Remark~\ref{rmrk1}), we arrive at the following result:
\begin{proposition}
    The set of states $\mathcal{X}_\text{singular} \subset \mathcal{X}$ satisfying equation \eqref{infeasibility2} is a zero-volume set in $\mathcal{X}$ for $x \neq \hat{x}(\mathbf{x})$ when. That is $\mathcal{X}_\text{singular}$ is a null set.
\end{proposition}

\begin{proof}
Assume we have a solution $\tilde{\mathbf{x}} \in \mathcal{X}_\text{singular}$, denoted by $(x- \hat{x}(\tilde{\mathbf{x}})) \times R e_3 = q = 0^{3 \times 1}$.
For $\mathcal{X}_\text{singular}$ to be a zero volume set, we require $\frac{d}{dt} q \neq 0^{3 \times 1}$. Evaluating $\nabla_{r_{13}} q$, $\nabla_{r_{23}} q$ and $\nabla_{r_{33}} q$ with $R e_3 = [r_{13}, r_{23}, r_{33}]^T$ and $\kappa \rightarrow \infty$, we get 

\small
\begin{equation}\nonumber
\nabla_{r_{i3}} q = -e_i \times (x - \hat{x}(\mathbf{x})), \quad \forall i \in \{1,2,3\},
\end{equation}
\normalsize
which is nonzero if and only if $x \neq \hat{x}(\mathbf{x})$. In this case, one can always choose $\Omega$ such that $\frac{d}{dt} q \neq 0^{3 \times 1}$. This holds on $\mathcal{X}$ by defining $\nabla_{r} q$ on either side in discontinuities of $\hat{x}(\mathbf{x})$.
\end{proof}

The above result is important as it states that all possibly infeasible configurations form a null set, that can be left instantaneously by a suitable control action. In the infeasible set one of the constraints \eqref{composite_invariance} or \eqref{thrust_invariance} must be slackened to retain feasibility. 

\begin{figure*}[t]
    \centering
    \includegraphics[width=\linewidth]{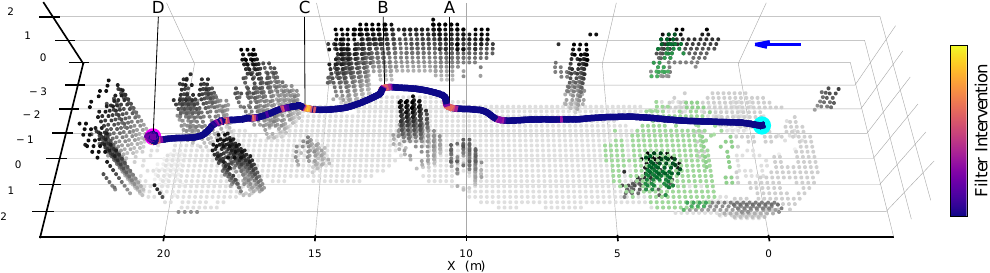}
    \caption{ \small Aggregated map and path of experiment A. An example of the obstacle map used by the safety filter is highlighted in green. The mission starts at the cyan circle on the right, receiving a constant velocity reference of $1 m/s$ in the positive $X$ direction, shown as a blue arrow. The norm of the intervention (cost of the \ac{qp} \eqref{OCP}) is color-coded into the path, highlighting the areas where the safety filter becomes active. Time instances marked A-D are reported in Fig. \ref{fig:CBF_elektro} for visualizing the corresponding numerical values.}
    \label{fig:mission_elektro}
    \vspace{-2ex}
\end{figure*}

\subsection{Safety Filter}
To integrate both constraints \eqref{composite_invariance} and \eqref{thrust_invariance} to jointly enforce invariance of $\mathcal{C}_x \cap \mathcal{C}_T $, we propose a safety filter, where the joint satisfaction of all constraints can be guaranteed everywhere except in the zero volume set $\mathcal{X}_\text{singular}$. Furthermore, the proposed architecture makes use of both thrust and attitude for obstacle avoidance and can easily be applied to in highly cluttered environments with hundreds of obstacles at a low computational cost.

The safe control law uses the proposed geometric tracking controller \eqref{lee_attitude} - \eqref{lee_position} as a nominal control law, on top of which a reactive safety filter as in \cite{backstepping_CBF} is applied. The control law $\mathbf{u}_\text{safe}$ is sequentially computed as 

\small
\begin{equation}\label{ref_controller}
\mathbf{u}_\text{ref} = 
\begin{bmatrix}
    \Omega \\ 
    \tau
\end{bmatrix}
= 
\begin{bmatrix}
    -k_R \frac{1}{2}[R_d^tR - R^T R_d]\textsuperscript{$\wedge$} + R^T R_d \Omega_{d} \\ 
    k_T (T_d - T) + \dot{T}_d
\end{bmatrix},
\end{equation}
\vspace{-2ex}
\begin{subequations}\label{OCP}
    \begin{align}
    \mathbf{u}_\text{safe} &= \underset{\mathbf{u} \in \mathcal{U}}{\text{argmin}} \; (\mathbf{u} - \mathbf{u}_\text{ref})^T P (\mathbf{u} - \mathbf{u}_\text{ref}) \\
    \text{s.t.} \quad 
    &\begin{bmatrix} \label{stacked_constraints}
        \mathcal{L}_g h_1(\mathbf{x}) \\
        \mathcal{L}_g h_2(\mathbf{x})
    \end{bmatrix}
    \mathbf{u} \geq 
    \begin{bmatrix}
        -\mathcal{L}_f h_1(\mathbf{x}) -\alpha_1 h_1(\mathbf{x})\\
        b_2(\mathbf{x})
    \end{bmatrix},
    \end{align}
\end{subequations}
\normalsize
where we simply stacked constraints \eqref{composite_invariance} and \eqref{thrust_invariance} in a single \ac{qp}. The safety filter thus corrects the commands of the reference controller to satisfy the invariance constraints. The resulting output matches the reference value if the constraints allow, and is minimally altered while satisfying the constraints otherwise.


\section{Evaluation Studies}\label{sec:evaluation}
\begin{figure}
    \centering
    \includegraphics[width=\linewidth]{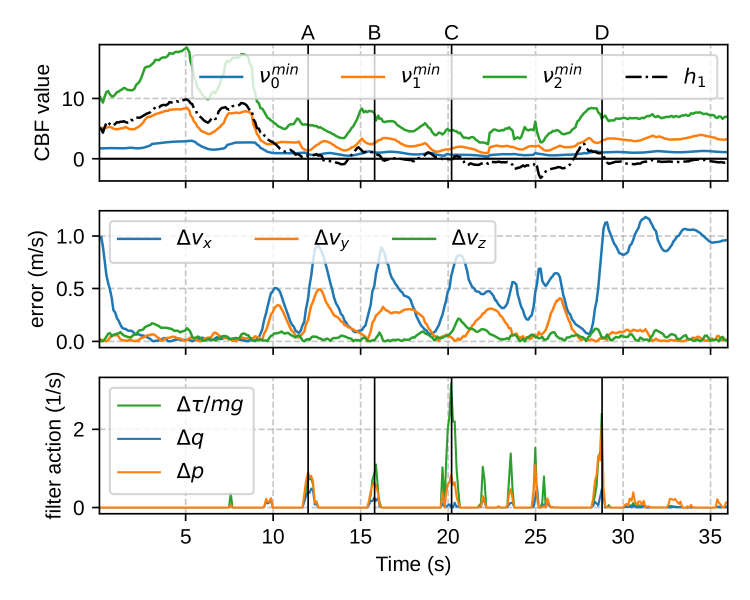}
    \vspace{-6ex}
    \caption{\small Constraint values (top) and velocity error w.r.t. the reference (center) and filter-induced deviation (bottom) during experiment A.}
    \label{fig:CBF_elektro}
    \vspace{-3ex}
\end{figure}

\subsection{Computational Scalability Study}
To evaluate the scalability of the proposed approach to cluttered environments with many obstacles, we evaluate the run times of the equations \eqref{composition1}, \eqref{composite_cbf} and \eqref{composite_invariance} for different environment sizes. Specifically, we ablate the required time for a single computation of $h(\mathbf{x})$, $\mathcal{L}_f h(\mathbf{x})$ and $\mathcal{L}_g h(\mathbf{x})$ in \eqref{composite_invariance} over different numbers of randomly placed obstacles. This is performed for the cases of computation of $\mathcal{L}_f h(\mathbf{x})$, $\mathcal{L}_g h(\mathbf{x})$ via the analytic gradients and via automatic gradient computation. Also, we evaluate this on a commercial laptop PC on CPU (Intel i7) and GPU (Nvidia RTX 3070Ti). All computations are performed in PyTorch. The processing times are displayed in table \ref{simulation_table}. Note that for fewer obstacles ($<10^3$) the processing time on CPU is lower than on GPU, while for more than $10^3$ obstacles, the computation benefits from the parallelism on the GPU.
These results show that the approach enables constructing the composite \ac{cbf} at high rates for environments with thousands of obstacles.

\begin{table}[t]
    \caption{Composition time of \ac{cbf} and Lie derivatives in milliseconds.}
    \label{simulation_table}    
    \centering
    \begin{tabular}{@{}llcccccc@{}}
        \toprule
        & & $10^1$ & $10^2$ & $10^3$ & $5 \cdot 10^3$ & $10^4$ \\
        \midrule
        \multirow{2}{*}{CPU} & Analytic & \textbf{0.264} & \textbf{0.279} & \textbf{0.485} & 2.575 & 3.353 \\
                             & Numeric  & 0.910 & 0.996 & 1.589 & 5.169 & 6.655 \\
        \midrule
        \multirow{2}{*}{GPU} & Analytic & 0.614 & 0.544 & 0.711 & \textbf{0.772} & \textbf{0.908} \\
                             & Numeric  & 2.606 & 2.703 & 2.760 & 2.653 & 2.624 \\
        \bottomrule
    \end{tabular}
    \vspace{-2ex}
\end{table}

\begin{figure*}[ht]
    \centering
    \includegraphics[width=\linewidth]{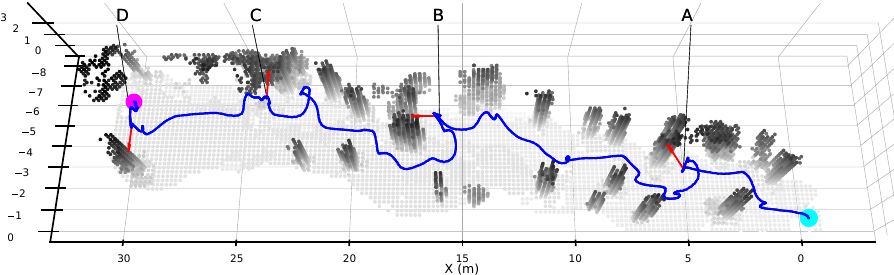}
    \caption{\small Aggregated map and path of experiment B. The mission starts at the cyan circle on the right. The quadrotor receives an adversarial velocity reference that actively tries to collide with obstacles. The red arrows depict this reference velocity for some selected time instances, A-D, which are reported in Fig. \ref{fig:CBF_dragvoll}.}
    \label{fig:mission_dragvoll}
    \vspace{-4ex}
\end{figure*}

\begin{figure}
    \centering
    \includegraphics[width=\linewidth]{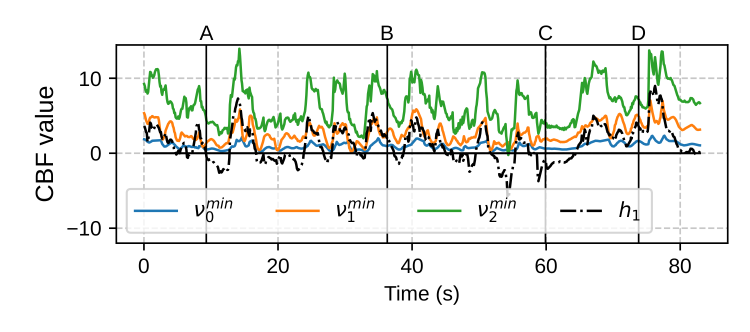}
    \vspace{-6ex}
    \caption{\small Constraint values during experiment B.}
    \label{fig:CBF_dragvoll}
    \vspace{-3.5ex}
\end{figure}

\subsection{Experimental Implementation}
The experiments relied on a custom-built quadrotor from \cite{harms2024neural} with dimensions $0.52\times 0.52\times0.31~\textrm{m}$ and a takeoff mass of $2.58~\textrm{kg}$. The system integrates PX4-based autopilot avionics for low-level control, together with an NVIDIA Orin NX single-board computer, as well as an Ouster OS0-64 LiDAR and a VectorNav VN-100 IMU used for odometry estimation as in~\cite{khattak2020complementary}. The mapping method \cite{voxblox} is then used to create a voxel-grid representation of the environment at a resolution of 20cm. The closest 400 occupied voxels are then used as the obstacles $x_i$ in \eqref{position_cbf} and updated at 10Hz.
The proposed safe control law \eqref{ref_controller} is implemented in Python, where we performed a minor adjustment by using the adaptive position control law introduced in~\cite{TalINDI}. This attenuates tracking errors introduced by modelling uncertainties in the thrust coefficient. The composition of the composite \ac{cbf} and its Lie derivatives (equations \eqref{composition1} - \eqref{composite_cbf}) uses PyTorch and the safety QP \eqref{OCP} is implemented in Casadi~\cite{Andersson_Casadi} using the QP solver qpOASES~\cite{Ferreau_qpOASES}. Furthermore, an estimate of the current thrust $T_\text{est}$ estimate is obtained by low-pass filtering the previous thrust commands. As a control output, the vector $[\Omega^T,T_\text{est}+ \tau \Delta t ]^T$ is sent to the autopilot, where $\Delta t$ is the time interval between control updates.
The safety filter and reference controller run on the Orin NX CPU with an update rate of 100Hz with the parameters listed in Tab. \ref{tab:parameters}.
Experiments A and B presented hereafter can be seen in the supplementary video.

\subsection{Experiment A}
In the first experiment, the quadrotor receives a constant reference velocity of $[1,0,0]\unit{m/s}$  and reference height of $1.3\unit{m}$ in an obstacle-filled hallway. The cluttered environment forces the safety filter to become active during many short periods, deflecting the quadrotor away from obstacles (points A and B), above an obstacle (point C), while finally reaching a dead-end where it remains in hover (point D). The trajectory and environment are shown in Fig. \ref{fig:mission_elektro}. The resulting set function values are plotted in Fig. \ref{fig:CBF_elektro}. It can be seen that all original constraints and higher order sets are satisfied over the entire mission since the minimum over constraints of $\nu_{0,i}$, $\nu_{1,i}$, $\nu_{2,i}$ is positive over the entire mission. The \ac{cbf} value is mostly positive but displays slight crossings of $h(\mathbf{x})=0$. These minor violations are small and are expected in a real system due to modelling errors and time-varying observations. The velocity tracking errors induced by the safety filter are also visualized in Fig. \ref{fig:CBF_elektro}, showing clearly that the corrections induced by the safety filter induce tracking errors. It can also be seen from Fig. \ref{fig:CBF_elektro} that the safety filter acts on roll-rate ($p$), pitch-rate ($q$) and thrust-rate ($\tau$) to achieve safe actions. The results demonstrate the ability of the proposed safe control law to enforce constraint satisfaction over the entire mission.

\begin{table}[t]
    \centering
    \caption{Parameter values used in the experiments.}
    \label{tab:parameters}
    \setlength{\tabcolsep}{5.5pt} 
    \renewcommand{\arraystretch}{1} 
    \begin{tabular}{c|cccccccc}
        \toprule
        \textbf{Parameter} & $p_0$ & $p_1$ & $\alpha_1$ & $\gamma$ & $\kappa$ & $\epsilon$ & $\alpha_2$  & $\epsilon_T$ \\
        \midrule
        \textbf{Value} & -3 & -2 & 1 & 40 & 20 & 0.5 & 5 & 7.5 \\
        \bottomrule
    \end{tabular}
    \vspace{-4ex}
\end{table}

\subsection{Experiment B}
The second experiment takes place in a forest with tree trunks
and foliage as natural obstacles. The reference velocity in this experiment is given by an adversarial operator, which intentionally attempts to collide the quadrotor with the surroundings. Horizontal and vertical references are given to provoke collisions with various trees and with the ground. A slight wind was present during the experiment, adding unmodelled disturbances. The trajectory and environment are shown in Fig. \ref{fig:mission_dragvoll}. The figure highlights four instances during the experiment where the unsafe references of the operator are corrected by the safety filter. The \ac{cbf} values are plotted in Fig. \ref{fig:CBF_dragvoll}, showing that the original constraints are all satisfied during the entire experiment. However, the value of $h_1$ often drops below the zero line due to the present disturbances. This experiment demonstrates the applicability of the proposed method in realistic challenging conditions.

\section{Conclusions}\label{sec:conclusions}
This work presented a novel approach for safe navigation of multirotors in unknown, cluttered environments.
The proposed safety filter leverages a Composite \ac{cbf} formulation for synthesizing a single \ac{cbf} from an arbitrary number of 1D collision constraints representing point-wise obstacles.
The resulting \ac{cbf} is both computationally scalable, and is shown to be recursively feasible, except for a zero-volume set of infeasible configurations.
The proposed method is validated in two hardware experiments, in varying conditions, against an environment-agnostic and an adversarial policy.
Future work includes extending the method to a robust \ac{cbf} design to explicitly account for modelling uncertainty, and investigating approaches to also account for constraints on the control input, while retaining feasibility.

\bibliographystyle{IEEEtran}
\bibliography{references}

\end{document}